\documentclass[10pt]{article} % For LaTeX2e

\usepackage[accepted]{rlj} % Should be uncommented for the camera-ready
\usepackage{amssymb}            % Defines common symbols like \mathbb R
\usepackage{mathtools}          % Extends amsmath, providing common math tools
\usepackage{mathrsfs}           % Enables \mathscr, which can work in cases that \mathcal does not
\mathtoolsset{showonlyrefs}     % Only number equations that are referenced (optional)
\usepackage{graphicx}           % For including images
\usepackage{tikz}
\usepackage[T1]{fontenc}

\usetikzlibrary{
    positioning,
    arrows.meta,
    shadows,
    fadings,
    shapes.misc
}
\definecolor{DeepTeal}{HTML}{00695C} % Dark Muted Teal
\definecolor{LightTeal}{HTML}{4DB6AC} % Lighter Accent Teal
\definecolor{SoftGray}{HTML}{EEEEEE} % Very Light Gray background
\definecolor{DarkText}{HTML}{333333} % Near-black text
\usetikzlibrary{positioning} 
\usepackage{subcaption}         % Allows for the use of subfigures and subcaptions
\usepackage[space]{grffile}     % For spaces in image names
\usepackage{url}                % For displaying URLs
\usepackage{lipsum}             % For placeholder text
\usepackage{amsmath,amssymb,amsthm}

\usepackage{hyperref}
\usepackage{booktabs}       % professional-quality tables
\usepackage{url}
\usepackage{booktabs}       % professional-quality tables
\usepackage{amsfonts}       % blackboard math symbols
\usepackage{nicefrac}       % compact symbols for 1/2, etc.
\usepackage{microtype}      % microtypography
\usepackage{xcolor}         % colors
\usepackage[english]{babel}
\usepackage{enumitem}
\usepackage{amsmath}
\usepackage{wrapfig}
\usepackage{booktabs}
\usepackage{wrapfig}
\usepackage{caption}
\usepackage{amssymb}
\usepackage{siunitx}
\usepackage{cleveref}
\usepackage{algorithm}
\usepackage{algpseudocode}
\usepackage{multirow}
\usepackage{enumitem}
\usepackage{float}
\usepackage{graphicx}
\usepackage{standalone}
\usepackage{pgfplots}
\usepackage{pgfplotstable}
\pgfplotsset{compat=1.18}
\usepgfplotslibrary{fillbetween}
\usepgfplotslibrary{groupplots}
\usepackage{mathtools}
\usetikzlibrary{matrix}
\usepackage{amsfonts}       % blackboard math symbols
\usepackage{nicefrac}       % compact symbols for 1/2, etc.
\usepackage{microtype}      % microtypography
\usepackage{colortbl}%

\newtheorem{theorem}{Theorem}[section]

\newtheorem{proposition}[theorem]{Proposition}

\newtheorem{assumption}[theorem]{Assumption}

\title{Learning in Low-Dimensional Subspaces:\\ Orthogonal Bottlenecks for Reinforcement Learning}

\setrunningtitle{Orthogonal Bottlenecks for Reinforcement Learning}

\author{Aleksandar Todorov\textsuperscript{1,$\star$}, Matthia Sabatelli\textsuperscript{1}}

\emails{a.todorov.4@student.rug.nl,
m.sabatelli@rug.nl}

\affiliations{
$^{1}$\textbf{University of Groningen, Groningen, The Netherlands}\\
\par % If including additional comments like below, use \par to add some whitespace. 
$^\star$ Corresponding author
}

\contribution{
    We provide a theoretical analysis showing that fixed orthogonal bottlenecks preserve expressivity and the induced optimization dynamics under a linear realizability assumption.
    }
    {
    Linear realizability assumptions are widely used in reinforcement learning theory to analyze value-based function approximation \citep{du_is_2020, weisz_online_2023}, but are rarely linked to architectural constraints in deep RL. We show that when the optimal value function is realizable by a linear map in feature space, inserting a fixed orthonormal bottleneck whose dimension meets or exceeds the intrinsic rank does not reduce representational capacity and yields learning dynamics equivalent to an explicit low-dimensional parameterization. This provides a principled justification for constraining learned representations via fixed orthogonal subspaces.
    }

\contribution{
    We empirically demonstrate that deep reinforcement learning representations can be compressed into low-dimensional orthogonal subspaces while preserving performance relative to the unconstrained baseline, and we analyze the resulting representation geometry.
    }
    {
    Across Classic Control, Atari, Brax MuJoCo, and multi-task Meta-World benchmarks, we find that baseline performance is typically recovered once the bottleneck dimension exceeds a small task-dependent threshold. In small domains, we visualize low-dimensional value manifolds; in larger-scale settings, we analyze performance trends and diagnostics such as effective rank. We further compare fixed and trainable projections, identifying regimes in which learning the projection, instead of imposing it, introduces instability and representation collapse.
    }

 % Your keywords

\summary{Deep reinforcement learning agents typically use high-dimensional neural representations, despite growing evidence that task-relevant value and policy structure may be intrinsically low-dimensional. We study a simple architectural prior that enforces such structure directly by inserting a fixed orthonormal projection between the encoder and downstream heads, constraining representations to operate within a low-dimensional subspace without auxiliary objectives, pretraining, or changes to the underlying RL algorithm. Under a linear realizability assumption, we show that once the bottleneck dimension exceeds the intrinsic rank of the optimal value function in feature space, the bottleneck preserves expressivity and does not alter the induced learning dynamics. Empirically, across single- and multi-task benchmarks, performance is typically recovered once the bottleneck dimension exceeds a small task-dependent threshold. In small domains, we visualize low-dimensional value manifolds, while in larger benchmarks, we observe sharp performance recovery as the bottleneck dimension increases. Diagnostic analyses further show that fixed orthogonal bottlenecks stabilize feature norms and are associated with higher effective rank, while learned projections can be unstable in some regimes. Together, these results suggest that deep reinforcement learning representations can often be faithfully compressed into low-dimensional orthogonal subspaces, and that fixed orthogonal bottlenecks offer a simple mechanism for shaping representation geometry. 
}

\begin{document}

\makeCover  % Create the cover page
\maketitle  % Make the title section

\begin{abstract}

Deep reinforcement learning (RL) agents commonly rely on high-dimensional neural representations, despite growing evidence that task-relevant value and policy structure may be intrinsically low-dimensional. In this work, we present a simple yet effective representation-level prior that inserts a fixed orthonormal projection to constrain encoder features to a low-dimensional subspace, requiring no auxiliary objectives, pretraining, or changes to the underlying RL algorithm. Under a linear realizability assumption, we prove that when the bottleneck dimension exceeds the intrinsic rank of the optimal value function in feature space, the bottleneck preserves expressivity and leaves the induced gradient dynamics unchanged up to an equivalent low-dimensional parameterization. Empirically, we find that across both single and multi-task benchmarks, baseline performance is either matched or improved once the bottleneck dimension exceeds a small task-dependent threshold; in many cases, value representations can be compressed to extremely low dimensions without loss, and the minimal sufficient dimension depends far more on environment complexity than encoder width. In addition, we analyze representation geometry and find that orthogonal bottlenecks stabilize feature norms and are associated with higher effective rank. Together, these results support a representation-space interpretation of the manifold hypothesis in reinforcement learning and position orthogonal bottlenecks as a lightweight, architecture-agnostic mechanism for shaping RL representations.

\end{abstract}

%%%%%%%%%%%%%%%%%%%%%%%%%%%%%%%%%%%%%%%%%%%%%%%%%%%%%%%%%%%%%%%%
%% Section: Submission of papers to RLJ/RLC
%%%%%%%%%%%%%%%%%%%%%%%%%%%%%%%%%%%%%%%%%%%%%%%%%%%%%%%%%%%%%%%%

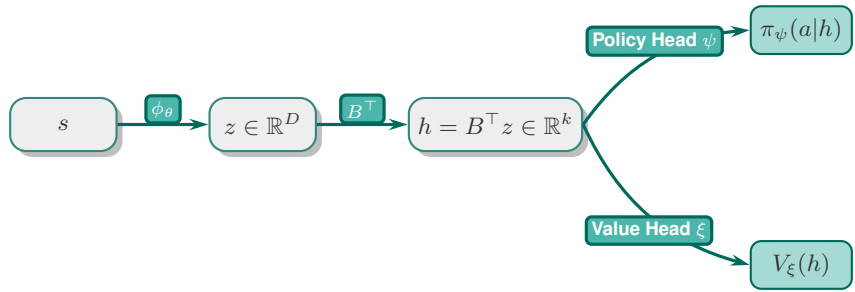
\begin{figure}[h]
\centering
\begin{tikzpicture}[
    % --- Global Styling ---
    >={Stealth[length=2.5mm, width=1.5mm]}, 
    line width=0.8pt,                       
    node distance=1.2cm, % Tight horizontal spacing for s -> z -> h
    auto
]

% --- Node Styles (Kept small for width control) ---
\tikzset{
    proc_block/.style={
        rectangle,
        rounded corners=6pt, 
        draw=DeepTeal!80,
        fill=SoftGray,
        drop shadow,
        text=DarkText,
        minimum height=2.0em, 
        minimum width=4.0em,
        align=center,
        font=\small\bfseries 
    },
    output_head/.style={
        rectangle,
        rounded corners=4pt,
        draw=DeepTeal,
        fill=LightTeal!50,
        text=DarkText,
        minimum height=1.8em,
        minimum width=3.8em,
        align=center,
        font=\small\bfseries
    },
    arrow_label/.style={
        rectangle,
        fill=LightTeal,
        rounded corners=2pt,
        draw=DeepTeal,
        text=white,
        font=\sffamily\scriptsize\bfseries,
        inner sep=2pt
    }
}

% --- Define the Shared Nodes (Strictly Horizontal) ---
\node[proc_block] (S) {$s$};
\node[proc_block, right=of S] (Z) {$z \in \mathbb{R}^D$};
\node[proc_block, right=of Z] (H) {$h = B^\top z \in \mathbb{R}^k$};

% --- Define the Split Policy and Value Heads (Increased X-Shift for Policy) ---
\def\VertSep{2.5cm} 
% Policy Head (top) - Increased xshift to 1.0cm
\node[output_head, right=of H, yshift=\VertSep/2, xshift=1.0cm] (Policy) {$\pi_\psi(a|h)$};
% Value Head (bottom)
\node[output_head, below=of Policy, yshift=-\VertSep/2] (Value) {$V_\xi(h)$};

% --- Draw the Shared Arrows (Straight Horizontal) ---
\draw[->, DeepTeal, line width=1.2pt] (S) -- node[arrow_label] {$\phi_\theta$} (Z);
\draw[->, DeepTeal, line width=1.2pt] (Z) -- node[arrow_label] {$B^\top$} (H); 

% --- Draw the Split Arrows with Bends and Clear Labels ---
% Arrow to Policy Head: Bending up
\draw[->, DeepTeal, line width=1.2pt] (H.east) to [bend left=25] 
    node[arrow_label, pos=0.5, above] {Policy Head $\psi$} (Policy.west);

% Arrow to Value Head: Bending down
\draw[->, DeepTeal, line width=1.2pt] (H.east) to [bend right=25] 
    node[arrow_label, pos=0.5, below] {Value Head $\xi$} (Value.west);

\end{tikzpicture}
\caption{A simple, visual representation of the orthogonal bottleneck for deep reinforcement learning adapted to a typical Actor-Critic architecture. After encoding the state into features $z \in \mathbb{R}^D$, a fixed orthonormal projection constrains the representation to a $k$-dimensional subspace before feeding the policy and value heads.}
\label{fig:orthogonal_bottleneck}

\end{figure}

% While deep networks offer the flexibility to model complex policies and value functions, this generality comes at a steep price, with modern deep RL agents often requiring millions of environment interactions to learn behaviors that are stable and transferable \citep{li_deep_2018, henderson_deep_2019, badia_agent57_2020, wang_deep_2024}. The problem becomes even more pronounced in multi-task settings, where a single agent must accommodate multiple objectives within a shared representational space \citep{yu_meta-world_2021, yang_multi-task_2020}. 

\newpage
\section{Introduction}\label{sec:introduction}
Reinforcement learning (RL) agents are routinely equipped with highly over-parameterized neural representations, even when solving tasks whose underlying decision structure is comparatively simple. While deep networks offer the flexibility to model complex policies and value functions, it is natural to question to what extent standard deep RL architectures allocate representational capacity far beyond what the task itself demands. This mismatch between network capacity and task complexity aligns naturally with the manifold hypothesis in machine learning, which posits that high-dimensional data and learned representations often concentrate near low-dimensional manifolds embedded in ambient space \citep{goldberg_manifold_2008, bengio_representation_2014,fefferman_testing_2016, meila_manifold_2023}. Recent evidence suggests that such low-dimensional structure also emerges in deep RL. On the policy side, \citet{mutti_reward-free_2022} and \citet{tenedini_parameters_2025} show that the space of behaviorally distinct policies realized by RL agents is effectively low-dimensional, allowing policy networks to be compressed by several orders of magnitude in parameter space without loss of behavioral expressivity. On the representation learning side, both model-free and model-based approaches aim to recover compact latent state manifolds that reflect task dynamics, typically through auxiliary or contrastive objectives that guide an encoder toward structured representations \citep{oord_representation_2019, zhang_learning_2021, echchahed_survey_2025}. These approaches, however, generally treat the manifold as a structure to be either discovered through optimization or to be recovered post hoc via generative modeling. 

In this work, we take a different approach and study an alternative perspective: rather than encouraging the network to uncover a low-dimensional structure, we impose one explicitly through the architecture itself at the representation level. Among the many possible ways to impose a low-dimensional bottleneck on learned representations, our focus is on fixed orthonormal subspaces as they allow for representations that capture the intrinsic subspace of a task without redundancy or distortion. In fact, any linear map onto a $k$-dimensional subspace can be factorized into an orthonormal basis followed by a diagonal scaling; retaining only the orthonormal component yields a full-rank bottleneck in which all directions are treated uniformly and no single eigendirection dominates the representation. Orthogonal projections also enjoy favorable geometric properties, with random orthogonal maps approximately preserving distances with high probability, as formalized by the Johnson-Lindenstrauss lemma \citep{jl-lemma,tipping_probabilistic_1999, ghojogh_factor_2022}, and orthogonal weight matrices are known to improve conditioning and gradient propagation in deep networks through non-expansiveness \citep{saxe_exact_2014, hu_provable_2019}. Motivated by these observations, we study a simple architectural inductive bias, namely, after an encoder produces features $z \in \mathbb{R}^D$, we project them onto a fixed orthonormal basis $B \in \mathbb{R}^{D \times k}$ with $B^\top B = I_k$ and $k \leq D$ via $h = B^\top z$,
and feed only the compressed representation $h \in \mathbb{R}^k$ to all downstream value and/or policy heads. This explicitly constrains the agent to operate within a $k$-dimensional orthogonal representation subspace, without altering the underlying learning algorithm or training objective. \Cref{fig:orthogonal_bottleneck} provides a visual overview of this architecture in a standard actor-critic setting.

Our contributions are as follows. First, we provide theoretical guarantees showing that if the optimal value function is linearly realizable in feature space with intrinsic rank $r$, then a fixed orthogonal bottleneck of dimension $k \geq r$ preserves expressivity and does not alter the induced gradient dynamics of the effective feature-to-representation mapping. Second, we empirically validate this approach across both single-task and multi-task benchmarks and multiple deep RL algorithms. We show that once $k$ exceeds a small task-dependent threshold, low-dimensional orthogonal subspaces typically match (and sometimes improve upon) baseline performance, while yielding more stable and uniformly utilized representations as measured by diagnostics such as feature norms and effective rank.

\section{Related Work}\label{sec:background-related-work}

Our work connects to two broad lines of research. The first studies low-dimensional structure in reinforcement learning, either in policy space or in learned state and value representations, motivated by the manifold hypothesis. The second investigates the role of orthogonality in deep networks as a mechanism for stabilizing signal and gradient propagation. While both bodies of work provide important insights into representation structure and training dynamics, they have largely been explored independently. We combine these perspectives by inserting a fixed orthonormal projection between the encoder and the policy and value heads. In turn, this makes the dimension available to the heads a controlled architectural parameter, allowing us to study its effect without auxiliary objectives, pretraining, or changes to the underlying reinforcement learning algorithm.

\paragraph{Low-dimensional structure and representations.}
A large body of work in reinforcement learning assumes that task-relevant variability is intrinsically low-dimensional and exploits this structure to design compact policy and state representations, often motivated by the manifold hypothesis \citep{narayanan_sample_2010, fefferman_testing_2016}. Unsupervised skill- and option-discovery methods explicitly construct low-dimensional latent spaces of behaviors or tasks using information-theoretic or variational objectives to encourage diversity and structure in the learned policy space \citep{frans_meta_2017, hausman_learning_2018, achiam_variational_2018, eysenbach_diversity_2018, laskin_unsupervised_2022}. More recent work makes a policy-manifold perspective explicit by learning generative models over policy parameters and compressing the policy space into low-dimensional latent codes \citep{rakicevic_policy_2021, mutti_reward-free_2022, tenedini_parameters_2025}. Parallel lines of research in model-based and transfer reinforcement learning learn latent state, action, or policy embeddings to simplify control and enable fast adaptation across related tasks \citep{zhang_solar_2019, arnekvist_vpe_2019, rana_residual_2022}. State (and state-action) representation learning methods similarly aim to recover compact, task-relevant features, typically through contrastive, predictive, or reconstruction-based objectives \citep{oord_representation_2019, zhang_learning_2021, echchahed_survey_2025}. On the theory side, work on linearly realizable value functions and good feature representations formalizes a related notion of low-dimensional linear structure in feature space, primarily from a worst-case sample complexity perspective \citep{du_is_2020, lattimore_learning_2020, weisz_online_2023}. Finally, empirical studies of deep RL representations link geometric properties such as feature rank, isotropy, and neuron dormancy to loss of plasticity and poor long-term learning, highlighting how collapsed or unstable representations impede adaptation \citep{lyle_understanding_2023, sokar_dormant_2023, klein_plasticity_2024, todorov_sparsity-driven_2025}.

\paragraph{Orthogonality and signal propagation.}
Orthogonal and orthonormal weight matrices have long been studied as a means of stabilizing signal and gradient propagation in deep networks \citep{xiao_dynamical_2018, jia_orthogonal_2019, yang_mean_2019, huang_controllable_2020}. In deep linear networks, \citet{saxe_exact_2014} and \citet{hu_provable_2019} show that orthogonal initialization can provably accelerate gradient descent compared to Gaussian initialization, and related techniques combining orthogonality with layerwise variance normalization enable reliable training of very deep convolutional architectures \citep{mishkin_all_2016}. In nonlinear feedforward networks, work on dynamical isometry demonstrates that when the singular values of the input-output Jacobian remain close to one, which is approximately achieved by random orthogonal weights in wide networks, gradients neither vanish nor explode, leading to substantially faster training \citep{pennington_resurrecting_2017, pennington_emergence_2018, xiao_dynamical_2018}. Closely related ideas have also been applied to recurrent architectures, where orthogonal or unitary transition matrices help preserve gradient norms over long sequences and improve learning of long-term dependencies \citep{henaff_recurrent_2016, arjovsky_unitary_2016, chen_dynamical_2018, gilboa_dynamical_2019}.

\section{Orthogonal Bottlenecks Preserve Expressivity and Optimization}
\label{sec:theory}

We start by studying the representational and optimization properties induced by orthogonal bottlenecks from a theoretical perspective. Our goal is to understand whether inserting a fixed $k$-dimensional orthonormal projection between the encoder and downstream heads preserves expressivity, and whether learning with such a projection alters the gradient dynamics compared to directly training a $k$-dimensional representation. To this end, we analyze the value estimation problem underlying a broad class of modern RL algorithms under a linear realizability assumption. This setting allows us to make precise statements about representational sufficiency and optimization behavior, and provides a way of interpreting empirical behavior in terms of low-rank linear structure in the learned feature space. We consider the standard supervised regression formulation of value learning, which serves as an inner loop for many reinforcement learning algorithms. For definitions on reinforcement learning and function approximation, see \Cref{app:theorem_proof}.

We assume access to a sufficiently rich encoder
\begin{equation*}
    \phi : \mathcal{S} \to \mathbb{R}^{D},
\end{equation*}
which maps states $s \in \mathcal{S}$ to a high-dimensional feature space. The target function (e.g., the optimal value function $V^\star$) is assumed to be linear in this feature space. 
\begin{assumption}[Linear realizability]
\label{ass:linear}
There exists a matrix $\Theta^{\star} \in \mathbb{R}^{m \times D}$ such that for all $s \in \mathcal{S}$,
\begin{equation*}
    V^{\star}(s) = \Theta^{\star} \phi(s).
\end{equation*}
\end{assumption}
The linear realizability assumption is standard in reinforcement learning theory \citep{lattimore_learning_2020, du_is_2020, weisz_online_2023} and formalizes the idea that a sufficiently expressive encoder can transform raw observations into a feature space where values are approximately linear. Importantly, this assumption does not require the overall network to be linear, as the encoder $\phi$ may be arbitrarily deep and nonlinear, and the heads following the bottleneck may also be nonlinear. The assumption only concerns the existence of a linear representation of $V^{\star}$ in feature space.

Now, let $z \in \mathbb{R}^D$ denote the encoder output, and let $B \in \mathbb{R}^{D \times k}$ be a matrix with orthonormal columns, i.e., $B^\top B = I_k$ with $k \le D$. The bottleneck representation is defined as $h = B^\top z \in \mathbb{R}^k$, and only $h$ is provided as input to the value and policy heads.

For value learning, we consider a generic differentiable function approximator
\begin{equation*}
    H(h;\theta) \approx V(s),
\end{equation*}

where $\theta$ collects all parameters after the bottleneck, $H: \mathbb{R}^k \to \mathbb{R}^m$ is expressive enough to realize at least a linear layer, as is the case for standard multi-layer perceptron (MLP) heads, and $V$ denotes the value function. A natural question is then: for which values of $k$ does this architecture retain the ability to represent $V^{\star}$, and does the presence of the fixed projection $B$ alter the optimization dynamics relative to directly learning a $k$-dimensional representation? These questions are addressed by the following proposition. The full proof is deferred to \Cref{app:theorem_proof}.

\begin{proposition}
\label{th-main:orthogonal-bottleneck}
Assume $V^{\star}$ is linearly realizable in feature space with rank $r = \mathrm{rank}(\Theta^{\star})$, and let $H(h;\theta)$ be any head expressive enough to realize at least a linear layer. For any $k \ge r$ and orthonormal $B \in \mathbb{R}^{D \times k}$:
\begin{enumerate}[leftmargin=*,itemsep=2pt]
    \item \textbf{(Representational sufficiency).}
    There exist encoder parameters and head parameters $\theta^{\star}$ such that the network
    \[
        s \;\mapsto\; H\bigl(B^\top z(s);\theta^{\star}\bigr)
    \]
    exactly realizes $V^{\star}(s)$ for all $s \in \mathcal{S}$. In particular, once $k \ge r$, inserting a fixed orthogonal bottleneck does not reduce expressivity relative to the given feature space.
    
    \item \textbf{(Trainability):} Let $W \in \mathbb{R}^{D \times D}$ be the encoder's final layer and $A_t = B^\top W_t$ the composite feature-to-bottleneck map. Training $(\theta, W)$ by gradient descent on loss $\mathcal{L}$ evolves $A_t$ identically to training the direct parameterization $h = C\phi(s)$ on $(\theta, C)$, given $C_0 = A_0$.
\end{enumerate}
\end{proposition}

\Cref{th-main:orthogonal-bottleneck} establishes that once $k$ matches the intrinsic rank, expressivity is preserved, and orthogonality avoids introducing an additional linear preconditioner into the updates of the feature-to-bottleneck map. In addition, sharing a single orthogonal bottleneck across all value and/or policy heads constrains different learning objectives to operate within the same $k$-dimensional subspace, rather than allowing each head to learn an independent low-rank projection. This encourages a common representation geometry across objectives. 

%Finally, orthogonality ensures that the bottleneck acts as an isometry on its span and does not introduce additional conditioning or gradient pathologies.

\paragraph{Why Orthogonality Matters.} 

\begin{wrapfigure}{r}{0.48\textwidth}
\vspace{-6pt}
\centering
\includegraphics[width=0.48\textwidth]{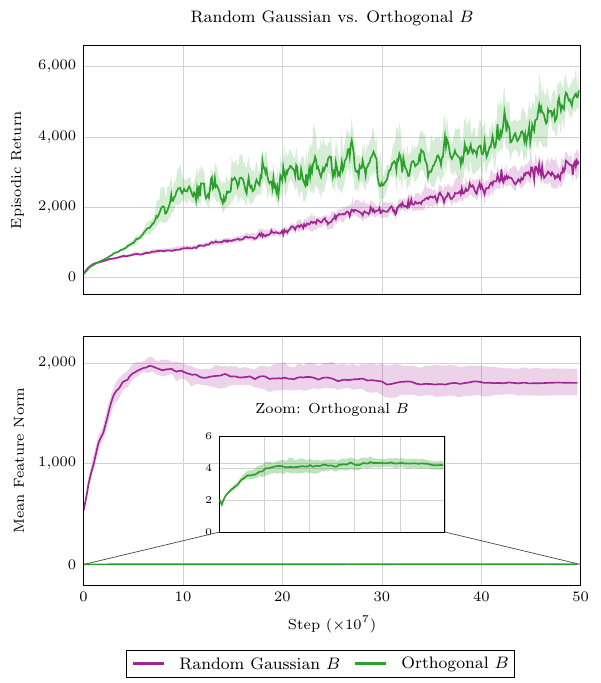}
\vspace{-10pt}
\caption{For PPO in Humanoid, feature norms explode when using a fixed Gaussian projection $B$ and lower performance is achieved, while a fixed orthonormal $B$ learns reliably. Both bottlenecks use $k=8$.}
\label{fig:ortho_vs_gauss}
\vspace{-8pt}
\end{wrapfigure}
A key requirement in \Cref{th-main:orthogonal-bottleneck} is that the projection matrix satisfies the orthogonality condition $B^\top B = I_k$. In particular, it ensures that the induced gradient dynamics on the effective bottleneck map $A_t = B^\top W_t$ are identical to those of a standard $k$-dimensional parameterization. If $B$ is fixed but non-orthogonal, the update for $A_t$ is instead preconditioned by $B^\top B$, which can amplify dominant singular directions and lead to unstable scaling. To illustrate why this matters in practice, \Cref{fig:ortho_vs_gauss} compares a fixed orthonormal projection to a fixed projection sampled from a standard Gaussian distribution of the same bottleneck dimension. While both bottlenecks impose the same dimension $k$, the Gaussian projection exhibits rapidly growing feature norms and degraded performance, whereas the orthonormal projection trains reliably with stable feature scales. This also provides empirical support for \Cref{ass:linear} (linear realizability). If the optimal value function were not well-approximated by a low-rank linear map in the learned feature space, then small bottleneck dimensions would necessarily limit performance. In the following sections, we interpret the recovery threshold in $k$ as an empirical falsification test: across environments, performance is preserved once $k$ exceeds a small task-dependent value, consistent with the presence of approximately low-rank linear structure in learned value representations.

\section{Experimental Setup}\label{sec:experimental-setup}
We evaluate orthogonal projection as a representation prior across a diverse suite of environments and algorithms with the goal of testing how performance depends on the bottleneck dimension $k$, and to empirically probe the theoretical predictions of \Cref{sec:theory} regarding representational sufficiency and the role of orthogonality. We compare a standard baseline agent with no bottleneck to variants that insert a $k$-dimensional projection, varying the bottleneck dimension and whether the projection matrix $B$ is fixed orthogonal or trained end-to-end. In selected experiments, we additionally vary the encoder width $D$ while keeping $k$ fixed to assess the role of encoder capacity once the representation dimension is constrained.

\subsection{Environments and Algorithms}
We consider five families of environments spanning increasing perceptual and control complexity: Classic Control \citep{towers_gymnasium_2025}, MinAtar \citep{young_minatar_2019}, Atari \citep{bellemare_arcade_2013}, Brax MuJoCo \citep{todorov_mujoco_2012, freeman_brax_2021}, and Meta-World MT10 \citep{yu_meta-world_2021}. These domains allow us to study representation compression in low-dimensional state spaces, pixel-based environments, continuous control, and multi-task learning. Across these environments, we use standard model-free deep RL algorithms appropriate to each setting: DQN \citep{mnih_human-level_2015} for Classic Control, PQN \citep{gallici_simplifying_2025} for Atari, and PPO \citep{schulman_proximal_2017} for MinAtar, MuJoCo, and Meta-World.

\subsection{Architecture and Implementation}
All agents share a common architecture of the form as shown in \Cref{fig:orthogonal_bottleneck}, where $\phi_\theta$ is a convolutional encoder for MinAtar and Atari and an MLP encoder for all other environments. The matrix $B \in \mathbb{R}^{D \times k}$ is sampled once by QR decomposition of a Gaussian matrix and held fixed throughout training, with only the encoder and heads updated. This enforces a strict $k$-dimensional information bottleneck without introducing additional trainable parameters. The projection requires $\mathcal{O}(Dk)$ operations for the forward pass and a similar cost for backpropagation, and for typical configurations $(D=256, k=4)$, this is negligible compared to evaluating the encoder or interacting with the environment.

\subsection{Training and Evaluation Protocol}
Each configuration is run with 10 random seeds. For all experiments, we report the interquartile mean (IQM) of returns with 95\% stratified bootstrap confidence intervals, following the recommendations of \citet{agarwal_deep_2021}. To ensure a fair comparison, we use the best hyperparameters found for the unconstrained baseline agent. Complete hyperparameters, training budgets, and diagnostic definitions are provided in \Cref{supp:hyperparameters}.

% Finally, to study a ``useful'' subspace dimension, we examine performance as a function of $k$. For each environment family $\mathcal{E}$ and algorithm $A$, we define the critical dimension
% \begin{equation*}
%     k^* (\mathcal{E}, A) = \min \left \{ k: \frac{\mathcal{J}_\mathrm{proj} (\mathcal{E}, A, k)}{\mathcal{J}_\mathrm{base} (\mathcal{E}, A)} \geq \alpha \right \},
% \end{equation*}
% where $\mathcal{J}_\mathrm{proj}$ and $\mathcal{J}_\mathrm{base}$ are mean final returns and $\alpha$ is a fixed threshold (e.g., 95\% of baseline performance). We then analyze how $k$ changes across the previously mentioned environments and how it scales when we vary the encoder width $D$.

\section{Results}\label{sec:results}
We evaluate orthogonal bottlenecks across environments of increasing complexity, with the goal of understanding how small a representation subspace can be without sacrificing performance, and how orthogonal projections shape the geometry of learned value representations. We begin with small domains where representation structure can be visualized directly, then move to large-scale and multi-task benchmarks to test whether similar compressibility holds at scale.

\subsection{Classic Control and MinAtar: Small Bottlenecks and Value Manifolds}
Across Classic Control tasks, we find that a bottleneck of size $k=2$ is sufficient to recover baseline performance, while $k=1$ leads to clear degradation or failure. Increasing the bottleneck dimension beyond $k=2$ provides no additional benefit. MinAtar exhibits the same qualitative behavior: despite pixel-based observations, baseline performance is matched once $k \in \{1,2,3\}$, indicating that task-relevant value information lies in a very low-dimensional subspace. Full learning curves are provided in \Cref{supp:learning_curves}.
To understand how such small representations suffice, we visualize bottleneck activations collected along evaluation trajectories (greedy for DQN). \Cref{fig:classic_manifold_fixed_2d} shows two-dimensional embeddings for Acrobot-v1 (DQN) and Freeway-MinAtar (PPO) with $k=2$. In both cases, the embeddings concentrate on a thin, low-dimensional manifold rather than filling the ambient space. Coloring by the agents' value estimates reveals a smooth gradient along the manifold, while coloring by action reveals structured, partially overlapping regions on the manifold. Qualitatively, for Acrobot, a single episode trajectory progresses from low-value regions toward higher-value regions in a largely monotone fashion. For Freeway, trajectories are less monotone and repeatedly revisit parts of the manifold as the agent alternates between waiting, positioning, and crossing behaviors. This difference is consistent with the more reactive, cyclic dynamics of Freeway compared to the more goal-directed progression in Acrobot. When the bottleneck dimension is increased to $k=3$, the embeddings remain strongly concentrated near a two-dimensional structure in $\mathbb{R}^3$ (\Cref{fig:classic_manifold_fixed_3d}) with only modest thickness along the third coordinate. Notably, these additional variations are not aligned with the agent's value estimates, suggesting that the additional dimension is largely unused and that the intrinsic value representation remains effectively two-dimensional even when extra capacity is available.

\begin{figure}[htp]
    \centering
    \includegraphics[width=0.8\linewidth]{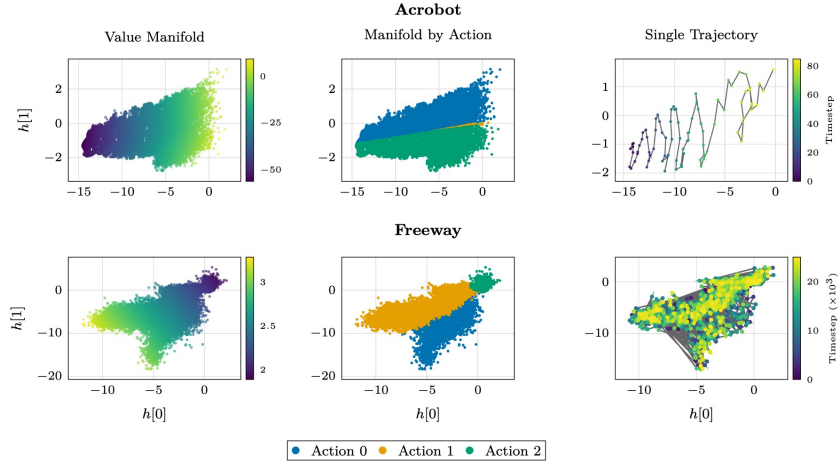}
    \caption{Bottleneck manifolds for Acrobot-v1 (DQN, top) and Freeway-MinAtar (PPO, bottom) with a fixed orthogonal bottleneck of size $k=2$. Each point is a visited state-action pair encoded into $(h[0],h[1])$, colored by the agent's value estimates (left) or action (middle); in both tasks, representations concentrate on a thin manifold with a smooth value gradient and structured action regions. The right column shows a single evaluation episode colored by timestep: Acrobot trajectories progress largely monotonically toward higher-value regions, while Freeway trajectories repeatedly revisit parts of the manifold reflecting more reactive task dynamics.}
    \label{fig:classic_manifold_fixed_2d}
\end{figure}

\begin{figure}[htp]
    \centering
    \includegraphics[width=0.87\textwidth]{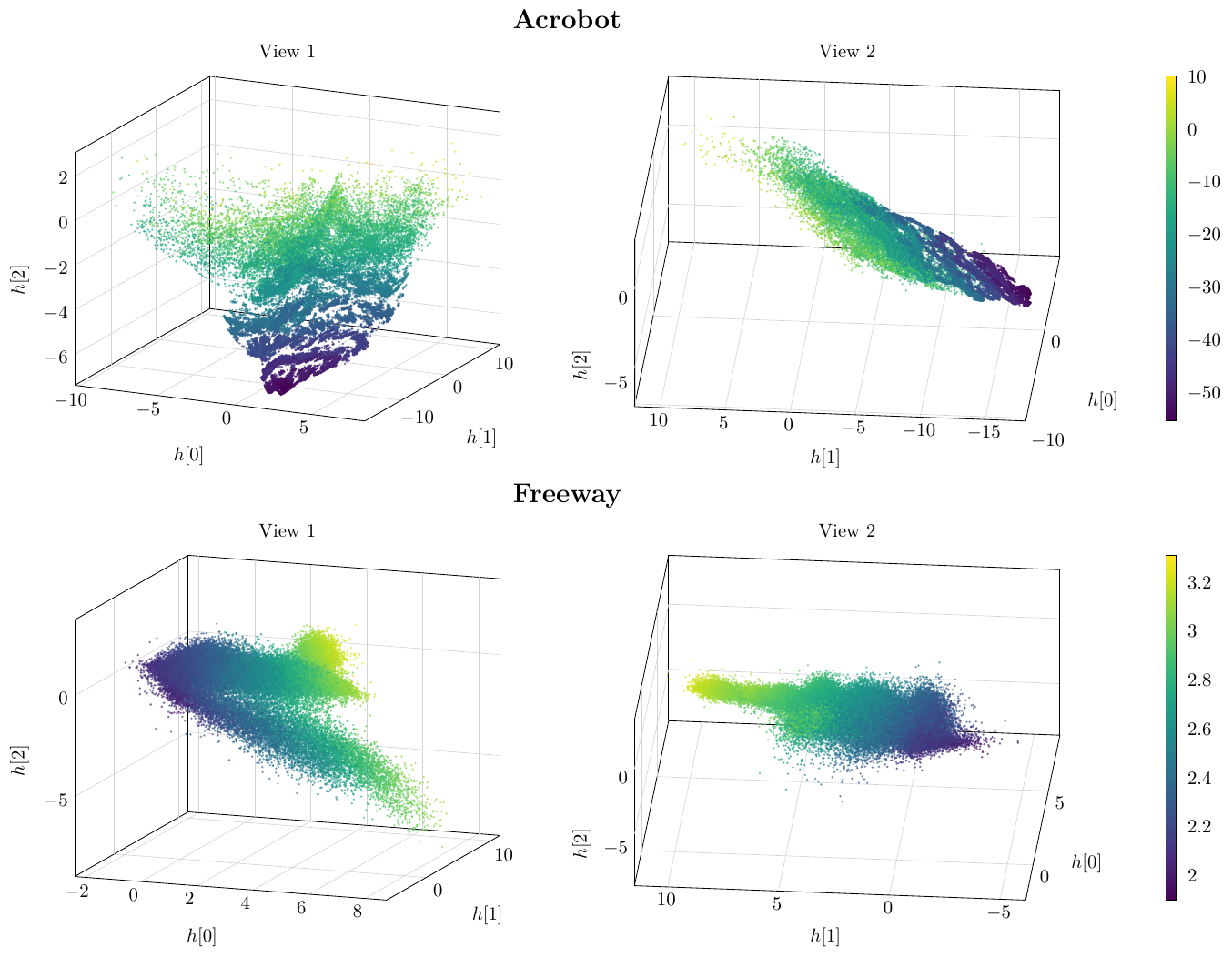}
    \caption{Three-dimensional bottleneck embeddings for Acrobot-v1 (DQN, top) and Freeway-MinAtar (PPO, bottom) with a fixed orthogonal bottleneck of size $k=3$, colored by the agent's value estimates. Each environment is shown from two viewing angles. In both cases, the representations concentrate near a low-dimensional structure in $\mathbb{R}^3$ rather than filling the volume, with slight thickness along the third coordinate. These additional variations are not strongly aligned with the value estimates, consistent with the observation that increasing $k$ beyond the recovery threshold does not change performance.
}
    \label{fig:classic_manifold_fixed_3d}
\end{figure}

\subsection{Large-Scale Benchmarks: Atari and Brax MuJoCo}
We next evaluate on the Atari-5 benchmark (Battle Zone, Double Dunk, Name This Game, Phoenix, and Q*bert) \citep{aitchison_atari-5_2022} and on four MuJoCo tasks (Reacher, Pusher, HalfCheetah, and Humanoid), chosen to span a range of perceptual and dynamical complexity.

\Cref{fig:atari_mujoco} reports final performance as a function of bottleneck dimension $k$. Across all tasks, learning fails when $k$ is too small but reliably recovers once it exceeds a modest, task-dependent threshold. Beyond this point, performance saturates and remains statistically indistinguishable from the no-projection baseline. While the precise value of $k_{\min}$ varies across tasks, reflecting differences in dynamical complexity and control requirements, in all cases, it remains small relative to the encoder width (256 in MuJoCo, 512 in Atari). In some environments, such as Reacher, moderate bottlenecks slightly outperform the baseline, whereas in others, performance saturates near the lower end of the baseline confidence interval. These results mirror the behavior observed in smaller domains: high-dimensional encoder features can be compressed into a low-dimensional orthogonal subspace without loss of performance, provided the bottleneck dimension exceeds the intrinsic rank of the task.

\begin{figure}[thp]
    \centering
    \includegraphics[width=\linewidth]{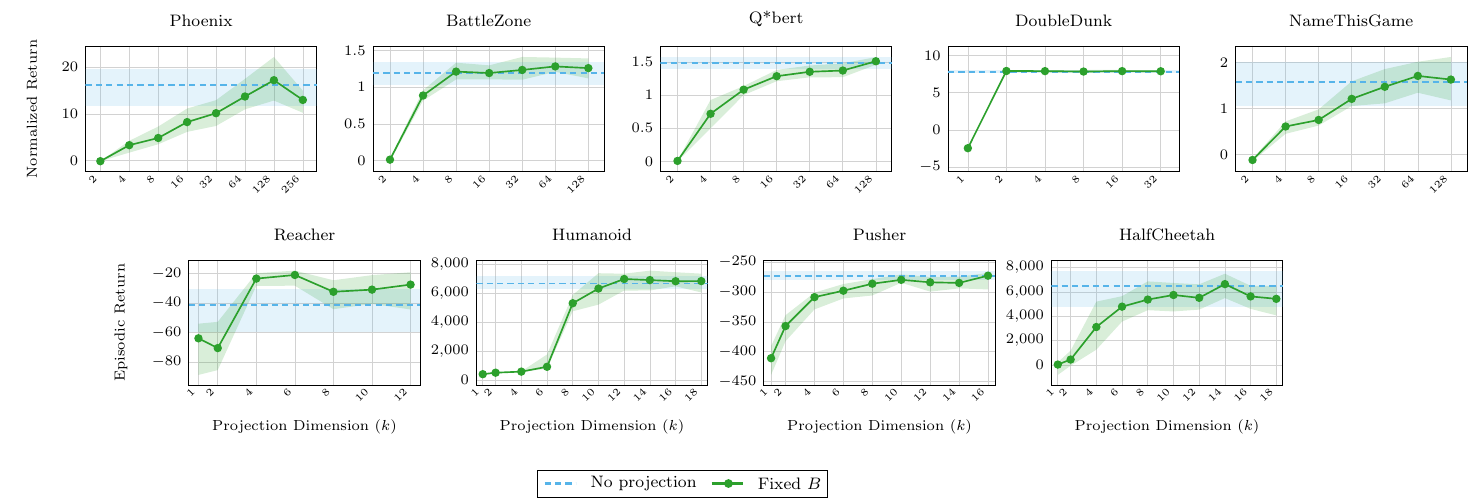}
    \caption{Final performance (IQM over seeds) as a function of bottleneck dimension $k$ for Atari (top) and MuJoCo (bottom) tasks. For each task, performance recovers once $k$ exceeds a small, task-dependent threshold and saturates thereafter, demonstrating that high-dimensional encoder features can be compressed into a very small subspace without loss of performance.}
    \label{fig:atari_mujoco}
\end{figure}

To further disentangle encoder capacity from representation dimensionality, we vary the width of the encoder's final layer on Humanoid while fixing the bottleneck dimension to $k=8$. Across a wide range of encoder widths, learning curves and final performance remain similar, indicating that bottleneck dimension, rather than encoder width, is the primary factor governing expressivity in this setting (see \Cref{fig:encoder_width} in \Cref{supp:encoder_width}).

\subsection{Trainable Projections and Representation Geometry}
We further compare fixed orthogonal projections against fully trainable end-to-end projection matrices. This comparison is intended to isolate the effect of the fixed projection assumed in \Cref{th-main:orthogonal-bottleneck}, assess whether learning the bottleneck subspace provides additional benefits, and examine how this choice interacts with representation geometry. \Cref{fig:trainable_vs_nontrainable} summarizes these results on two representative Atari (Phoenix) and MuJoCo (Humanoid) tasks, shown for both small and large bottleneck dimensions.

For Humanoid, allowing the projection to be trainable yields slightly higher returns at small bottleneck dimensions, but this advantage disappears as the bottleneck dimension increases. At larger $k$, fixed and trainable projections achieve similar performance. In contrast, for Phoenix, trainable projections exhibit unstable behavior at larger bottleneck dimensions and can collapse performance entirely, whereas fixed orthogonal projections remain reliable across both small and large $k$. Overall, this indicates that making the projection trainable might introduce additional sensitivity to the task and bottleneck size, while fixed orthogonal projections exhibit more consistent behavior across settings.

\begin{figure}
    \centering
    \includegraphics[width=1\linewidth]{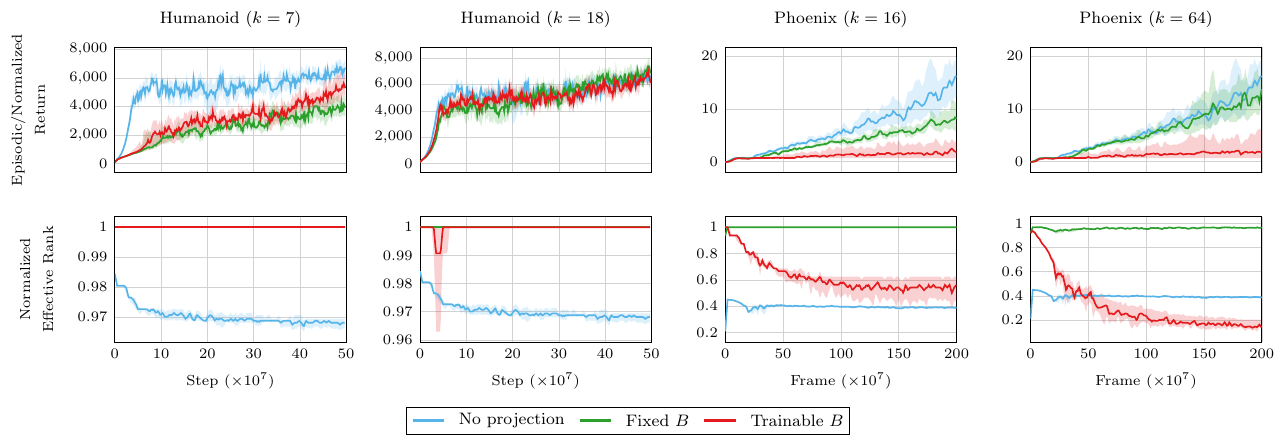}
    \caption{Performance (top row) and normalized mean effective rank (bottom row) for no projection, fixed orthogonal projection, and trainable projection on Humanoid and Phoenix at small and large bottleneck dimensions. Trainable projections can be beneficial or unstable depending on the environment, with performance degradation coinciding with severe rank collapse. Fixed orthogonal projections remain stable across tasks and bottleneck sizes.}
    \label{fig:trainable_vs_nontrainable}
\end{figure}

The corresponding effective-rank diagnostics in \Cref{fig:trainable_vs_nontrainable} help contextualize these performance differences. Effective rank measures how many dimensions of the representation are actively used, and prior work has linked rank collapse to loss of plasticity and degraded learning dynamics in deep reinforcement learning \citep{kumar_implicit_2021,lyle_understanding_2022,klein_plasticity_2024}. Fixed orthogonal projections consistently maintain high effective rank relative to their dimensionality, indicating uniform usage of the available subspace. In contrast, trainable projections and unconstrained baselines exhibit more variable rank dynamics, with effective rank often substantially lower than the ambient feature dimension and a higher prevalence of weakly used or inactive directions.

Importantly, low effective rank does not, by itself, imply learning failure: the baseline agents learn successfully despite exhibiting comparatively low effective rank. However, when learning becomes unstable or collapses, this failure is typically accompanied by a pronounced drop in the effective rank. This relationship is most clearly visible for Phoenix at large bottleneck dimensions, where the collapse of trainable projections coincides with a sharp reduction in effective rank, while fixed orthogonal projections at the same dimensionality maintain both stable learning and high rank.

\subsection{Multi-task Meta-World}
\begin{wrapfigure}{r}{0.48\textwidth}
\vspace{-12pt}
\centering
\includegraphics[width=\linewidth]{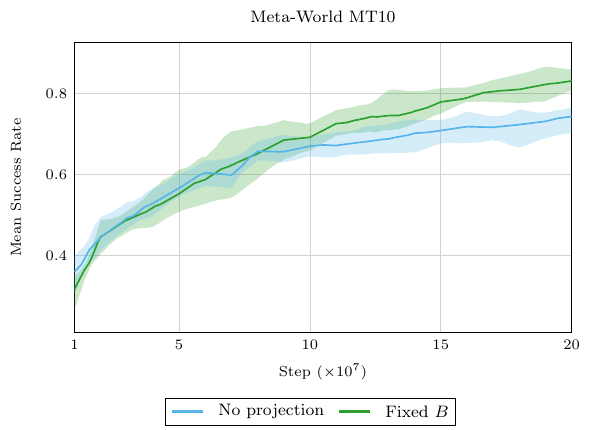}
\vspace{-10pt}
\caption{Meta-World MT10 performance of a baseline PPO agent and an agent equipped with a fixed bottleneck of dimension $k=24$.}
\label{fig:metaworld}
\vspace{-8pt}
\end{wrapfigure}

We finally turn to the multi-task setting, where a single agent must solve multiple tasks simultaneously using shared parameters. Multi-task RL is often limited by negative transfer and representational interference, since tasks compete for shared capacity and updates from one task can degrade performance on others, and such settings are known to be particularly sensitive to how representational capacity is allocated across tasks \citep{yu_gradient_2020, yang_multi-task_2020}. This issue has motivated a variety of methods in multi-task and continual learning that explicitly modify the geometry of the learning process, including PCGrad, which projects conflicting gradients \citep{yu_gradient_2020}, MOORE, which uses orthogonal expert representations to promote diversity and reduce interference \citep{hendawy_multi-task_2024}, and low-rank orthogonal subspaces for reducing task interference in continual learning \citep{chaudhry_continual_2020}.

\Cref{fig:metaworld} compares a standard PPO baseline against an agent equipped with a fixed orthogonal bottleneck of dimension $k=24$ on Meta-World MT10. With a suitably chosen bottleneck dimension, the orthogonal bottleneck modestly improves performance relative to the no-projection baseline. This result indicates that constraining the agent to operate within a shared low-dimensional subspace does not hinder multi-task learning and can, in fact, even be beneficial, suggesting that appropriately structured low-dimensional representations can be compatible with effective parameter sharing across tasks. We view this improvement over the unconstrained baseline as a secondary regularization effect, in line with a large body of prior work that has shown that structured constraints, such as sparsity, can improve learning by reducing interference or preserving plasticity \citep{graesser_state_2022, obando-ceron_value-based_2024, todorov_sparsity-driven_2025}. These findings are also consistent with our broader finding that orthogonal bottlenecks encourage uniformly utilized shared representations, which may help mitigate degenerate representation collapse in settings with competing objectives.

\section{Discussion and Conclusion}
Our results provide empirical evidence that many reinforcement learning tasks admit low intrinsic representation dimensions, even when solved with highly overparameterized neural architectures. Across our benchmarks, the minimal sufficient bottleneck dimension varies substantially across environments, but performance is comparatively insensitive to encoder width in our Humanoid sweep at a fixed $k$. This pattern is consistent with recent conjectures that intrinsic manifold dimensionality is driven primarily by environment complexity rather than network size, though \citet{tenedini_parameters_2025} formulate this for policy manifolds rather than value representations. In small domains, such structures can be visualized explicitly as thin value manifolds embedded in an ambient feature space. While these manifolds exhibit smooth and coherent geometry, they remain largely uninterpretable: the learned axes do not correspond to semantically meaningful factors of the environment. An interesting direction for future work is to connect these findings to object-centric reinforcement learning, where representations are explicitly structured around entities, relations, and compositional structure \citep{greff_multi-object_2019, locatello_object-centric_2020, haramati_entity-centric_2023}. Such approaches may offer more interpretable low-dimensional representations than those learned implicitly by standard encoders. Studying how structured bottlenecks interact with object-centric architectures could help determine whether the low-dimensional manifolds observed here can be aligned with meaningful latent factors rather than remaining purely geometric.

From a theoretical perspective, the empirical success of small orthogonal bottlenecks provides indirect evidence that the linear realizability assumption commonly used in RL theory can be made to hold approximately by modern neural encoders. While this assumption cannot be verified directly, and alternative explanations cannot be ruled out, the fact that performance is preserved once the bottleneck dimension exceeds a small task-dependent threshold is consistent with the idea that neural networks can learn feature spaces in which the value or action-value function is well-approximated by a low-rank linear map. If such a structure were absent, the representational sufficiency guarantees implied by our analysis would fail, and small bottlenecks would necessarily degrade performance. That this does not occur across a wide range of tasks and algorithms suggests that deep RL representations may be more amenable to low-rank and linear analysis than is commonly assumed.

Overall, we show that the representations learned by modern deep RL agents can often be compressed into surprisingly low-dimensional subspaces without loss of performance. Orthogonal bottlenecks provide a lightweight and architecture-agnostic mechanism for exposing and exploiting this low-dimensional structure, offering a step toward more principled and geometry-aware representation design in reinforcement learning.

\section*{Acknowledgements}
The authors thank the Center for Information Technology of the University of Groningen for their support and for providing access to the Hábrók high-performance computing cluster. We also greatly appreciate everyone who was involved in providing feedback for the final version of the manuscript. Aleksandar is highly grateful for the financial support provided by the Department of Artificial Intelligence at the University of Groningen and his supervisor, Matthia Sabatelli, who allowed the execution and presentation of the current work to happen.

\newpage

%%%%%%%%%%%%%%%%%%%%%%%%%%%%%%%%%%%%%%%%%%%%%%%%%%%%%%%%%%%%%%%%
%% NOTE: THIS MARKS THE END OF THE "MAIN TEXT"
%%%%%%%%%%%%%%%%%%%%%%%%%%%%%%%%%%%%%%%%%%%%%%%%%%%%%%%%%%%%%%%%

\bibliography{main,references}
\bibliographystyle{rlj}

\beginSupplementaryMaterials
\appendix

\section{Implementation and Hyperparameters}\label{supp:hyperparameters}

In this appendix, we report the full hyperparameter configurations used in all experiments. Tables are grouped by benchmark family: Classic Control (CartPole-v1, Acrobot-v1), MinAtar, Atari, MuJoCo, and Meta-World. Unless otherwise stated, hyperparameters are shared across tasks within the same family and algorithm, and differences are explicitly noted in the table captions. The projection dimensions are chosen to be the smallest $k$ that matches baseline performance for each benchmark family. Moreover, as stated in the main text, we report the interquartile mean (IQM) of returns with 95\% stratified bootstrap confidence intervals, following the recommendations of \citet{agarwal_deep_2021}.

For Classic Control, we use the \texttt{purejaxrl} DQN implementation by \citet{lu_discovered_2022}. Similarly, for Atari, we use the \texttt{purejaxql} PQN implementation by \citet{gallici_simplifying_2025}. For MinAtar \citep{young_minatar_2019} and Brax MuJoCo \citep{freeman_brax_2021} we use PPO, and for Meta-World, we use the multi-task PPO implementation from \citet{mclean_meta-world_2025}. Hyperparameters were chosen slightly differently for each environment:
\begin{itemize}
    \item DQN in CartPole: the original implementation provided by \texttt{purejaxrl} provides an already-tuned algorithm.
    \item PPO in MinAtar: we tune the hyperparameters so that they reach (or exceed) the baseline PPO performance as in \citet{young_minatar_2019} and \citet{gallici_simplifying_2025}.
    \item PQN in Atari: the implementation by \citet{gallici_simplifying_2025} provides an already-tuned implementation. We follow the original PQN theory and architecture and apply LayerNorm after every layer. Accordingly, when inserting the orthogonal bottleneck, we apply LayerNorm after the projection to maintain consistent feature scaling.
    \item PPO in MuJoCo Brax: given that the network architecture with a shared encoder is non-standard for MuJoCo, we tune the algorithm until it reaches (or exceeds) the baseline Brax PPO performance as reported in \citet{freeman_brax_2021}.
    \item PPO in Meta-World: we use the standard Meta-World PPO hyperparameters, specified in \citet{mclean_meta-world_2025}. We found that the baseline PPO performance degrades significantly otherwise.
\end{itemize}

While many definitions for the effective rank exist in the literature, we consider the one commonly used in RL \citep{kumar_implicit_2021,obando-ceron_value-based_2024,todorov_sparsity-driven_2025}. Namely, given a batch of feature vectors collected during rollouts, we form a feature matrix $X \in \mathbb{R}^{N \times d}$ where each row is one feature vector. We center features across the batch,
$\tilde X = X - \mu$, where $\mu$ denotes the batch mean activation. Let $\sigma_1 \ge \cdots \ge \sigma_l$ be the singular values of $\tilde X$ (with $l=\min(N,d)$). We define normalized singular values
\[
p_i = \frac{\sigma_i}{\sum_{j=1}^l \sigma_j},
\]
and the mean effective rank as the smallest $k_{\mathrm{eff}}$ such that the cumulative mass exceeds $1-\delta$:
\[
k_{\mathrm{eff}} = \min\Bigl\{k \in \{1,\dots,l\} : \sum_{i=1}^k p_i \ge 1-\delta \Bigr\}.
\]
We use the standard value $\delta=0.01$ for all experiments and report the normalized effective rank $
k_{\mathrm{norm}} = \frac{k_{\mathrm{eff}}}{k}$, so that $k_{\mathrm{norm}} \in[0,1]$ measures the fraction of available bottleneck dimensions that are effectively used. In the absence of a bottleneck, $k$ corresponds to the layer width. The full code is available at \href{https://github.com/atodorov284/ortho-bottlenecks}{https://github.com/atodorov284/ortho-bottlenecks}.

\begin{table}[htbp]
  \centering
  \small
  \caption{Classic Control (CartPole-v1, Acrobot-v1) DQN hyperparameters. The same configuration is used for both environments.}
  \label{tab:hparams_classiccontrol_dqn}
  \resizebox{0.4\textwidth}{!}{
  \begin{tabular}{l l}
    \toprule
    \textbf{Hyperparameter} & \textbf{Value} \\
    \midrule
    Total timesteps & $5\times 10^{5}$ \\
    Training environments & 10 \\
    Discount $\gamma$ & 0.99 \\
    \midrule
    Learning rate & $2.5\times 10^{-4}$ \\
    Linear LR decay & False \\
    Replay buffer size & 10{,}000 \\
    Batch size & 128 \\
    Learning starts & 10{,}000 \\
    Training interval & 10 \\
    Target update interval & 500 \\
    Polyak $\tau$ & 1.0 \\
    \midrule
    $\varepsilon$-greedy start & 1.0 \\
    $\varepsilon$-greedy final & 0.05 \\
    $\varepsilon$ anneal timesteps & $2.5\times 10^{5}$ \\
    \midrule
    Encoder type & Linear \\
    Encoder feature dim ($D$) & 128 \\
    Encoder last dim & 128 \\
    Q-head hidden layers & 0 \\
    Q-head feature dim & 64 \\
    Activation & ReLU \\
    \midrule
    Testing environments & 128 \\
    Test rollout horizon & Episode length \\
    Number of metric evaluations & 100 \\
    Test $\varepsilon$ & 0.0 \\
    \bottomrule
  \end{tabular}}
\end{table}

\begin{table}[htbp]
\centering
\small
\caption{MinAtar PPO hyperparameters and evaluation settings.}
\resizebox{0.4\textwidth}{!}{
\begin{tabular}{@{}ll@{}}
\toprule
\textbf{Hyperparameter} & \textbf{Value} \\
\midrule
Environment & Breakout-MinAtar \\
Total timesteps & $1\times 10^{7}$ \\
Training environments & 64 \\
Rollout length (steps) & 128 \\
Update epochs & 4 \\
Minibatches per epoch & 8 \\
Discount factor $\gamma$ & 0.99 \\
GAE $\lambda$ & 0.95 \\
PPO clip $\epsilon$ & 0.2 \\
Entropy coefficient & 0.01 \\
Value coefficient & 0.5 \\
Max grad norm & 0.5 \\
Learning rate & $5\times 10^{-3}$ \\
Learning rate annealing & True \\
Activation & ReLU \\
Seed & 0 \\
\midrule
Encoder type & CNN \\
Encoder feature dim & 256 \\
Encoder last dim & 256 \\
Actor feature dim & 64 \\
Critic feature dim & 64 \\
Actor hidden layers & 0 \\
Critic hidden layers & 0 \\
\midrule
Testing environments & 128 \\
Test rollout horizon & Episode length \\
Number of metric evaluations & 100 \\
\bottomrule
\end{tabular}}
\label{tab:minatar_ppo_hparams}
\end{table}

\newpage

\begin{table}[htbp]
  \centering
  \small
  \caption{Atari (Battle Zone, Double Dunk, Name This Game, Phoenix, Q*bert) PQN hyperparameters and environment settings. Total timesteps correspond to $200$M frames with frame skip $4$.}
  \label{tab:hparams_atari_pqn}
  \begin{tabular}{l l}
    \toprule
    \textbf{Hyperparameter} & \textbf{Value} \\
    \midrule
    Total timesteps & $5\times 10^{7}$ \\
    Training environments & 128 \\
    Steps per env per update & 32 \\
    Epochs per update & 2 \\
    Minibatches per epoch & 32 \\
    Discount $\gamma$ & 0.99 \\
    $\lambda$ (trace / advantage) & 0.65 \\
    \midrule
    Learning rate & $2.5\times 10^{-4}$ \\
    Linear LR decay & False \\
    Max grad norm & 10 \\
    Normalization & LayerNorm \\
    Encoder last dim & 512 \\
    \midrule
    $\varepsilon$ start & 1.0 \\
    $\varepsilon$ final & 0.001 \\
    $\varepsilon$ decay ratio & 0.1 \\
    \midrule
    Episodic life & True \\
    Reward clip & True \\
    Sticky actions prob. & 0.0 \\
    Frame skip & 4 \\
    No-op max & 30 \\
    \midrule
    Testing environments & 8 \\
    Test $\varepsilon$ & 0.0 \\
    \bottomrule
  \end{tabular}
\end{table}

\newpage

\begin{table}[htbp]
  \centering
  \small
  \caption{Brax MuJoCo (PPO) hyperparameters for Reacher, Pusher, HalfCheetah, and Humanoid.}
  \label{tab:hparams_brax_ppo}
\resizebox{0.65\textwidth}{!}{
  \begin{tabular}{@{}lcccc@{}}
    \toprule
    \textbf{Hyperparameter}
      & \textbf{Reacher}
      & \textbf{Pusher}
      & \textbf{HalfCheetah}
      & \textbf{Humanoid} \\
    \midrule
    Total timesteps
      & $5\times10^{7}$
      & $5\times10^{7}$
      & $5\times10^{7}$
      & $5\times10^{7}$ \\
    Learning rate
      & $3\times10^{-4}$
      & $3\times10^{-4}$
      & $3\times10^{-4}$
      & $3\times10^{-4}$ \\
    Learning rate annealing
      & True & True & True & True \\
    Parallel environments
      & 2048 & 2048 & 2048 & 2048 \\
    Unroll length (steps)
      & 50 & 30 & 20 & 10 \\
    Update epochs
      & 8 & 8 & 8 & 8 \\
    Num.\ minibatches
      & 32 & 16 & 32 & 32 \\
    Discount $\gamma$
      & 0.95 & 0.95 & 0.95 & 0.97 \\
    GAE $\lambda$
      & 0.95 & 0.95 & 0.95 & 0.95 \\
    PPO clip $\epsilon$
      & 0.3 & 0.3 & 0.3 & 0.3 \\
    Entropy coefficient
      & $1\times10^{-3}$
      & $1\times10^{-2}$
      & $1\times10^{-3}$
      & $1\times10^{-3}$ \\
    Value loss coefficient
      & 0.5 & 0.5 & 0.5 & 0.5 \\
    Max grad norm
      & 0.5 & 0.5 & 0.5 & 0.5 \\
    Activation
      & Tanh & Tanh & Tanh & Tanh \\
    \midrule
    Action repeat
      & 4 & 1 & 1 & 1 \\
    Episode length
      & 1000 & 1000 & 1000 & 1000 \\
    Reward scaling
      & 5.0 & 5.0 & 1.0 & 0.1 \\
    Normalize observations
      & True & True & True & True \\
    \midrule
    Encoder type
      & Linear & Linear & Linear & Linear \\
    Encoder feature dim
      & 256 & 256 & 256 & 256 \\
    Encoder last dim
      & 256 & 256 & 256 & 256 \\
    Actor feature dim
      & 128 & 128 & 128 & 128 \\
    Num.\ actor layers
      & 2 & 2 & 2 & 2 \\
    Critic feature dim
      & 128 & 128 & 128 & 128 \\
    Num.\ critic layers
      & 2 & 2 & 2 & 2 \\
    Log-std init
      & 0.0 & 0.0 & 0.0 & 0.0 \\
    \bottomrule
  \end{tabular}}
\end{table}

\begin{table}[htbp]
\centering
\small
\caption{Meta-World MT10 training and architecture hyperparameters. The baseline agent uses the same hyperparameters but without a bottleneck.}
\resizebox{0.45\textwidth}{!}{
\begin{tabular}{@{}ll@{}}
\toprule
\textbf{Hyperparameter} & \textbf{Value} \\
\midrule
Terminate on success & False \\
Total environment steps & $2\times 10^7$ \\
Rollout steps per epoch & 10{,}000 \\
Evaluation frequency & $2000$ steps \\
Discount factor $\gamma$ & 0.99 \\
GAE $\lambda$ & 0.97 \\
Number of epochs & 16 \\
Gradient steps per epoch & 32 \\
Normalize advantages & False \\
Baseline type & MLP \\
\midrule
Policy network & $[400,400,400]$ \\
Value network & $[400,400,400]$ \\
Architecture & Vanilla MLP \\
Activation & \texttt{tanh} \\
Policy squashing & \texttt{squash\_tanh=False} \\
Optimizer & Adam \\
Learning rate & $3\times 10^{-4}$ \\
\midrule
Projector network width & 256 \\
Projector network depth & 3 \\
Bottleneck dimension & 24 \\
\bottomrule
\end{tabular}}
\label{tab:metaworld_hparams}
\end{table}

\newpage

\section{Expressivity Analysis}\label{app:theorem_proof}
This appendix section provides a self-contained proof of \Cref{th-main:orthogonal-bottleneck}, which formalizes two basic properties of fixed orthogonal bottlenecks under a linear realizability assumption.

We consider a discounted Markov decision process (MDP) $M=(\mathcal{S},\mathcal{A},\mathcal{P},\mathcal{R},\gamma)$ with state space $\mathcal{S}$, action space $\mathcal{A}$, transition kernel $\mathcal{P}(\cdot \mid s,a)$, reward function $\mathcal{R}(s,a)$, and discount factor $\gamma\in[0,1)$. A policy $\pi$ induces trajectories $(s_n,a_n,r_{n+1})_{n\ge0}$ with $a_n\sim\pi(\cdot\mid s_n)$, $s_{n+1}\sim\mathcal{P}(\cdot\mid s_n,a_n)$, and $r_{n+1}=\mathcal{R}(s_n,a_n)$. The objective is to maximize the expected discounted return $J(\pi)=\mathbb{E}_\pi[\sum_{n\ge0}\gamma^n r_{n+1}]$, which is commonly approached by estimating value functions. The state-value and action-value functions are
\[
V^\pi(s)=\mathbb{E}_\pi\!\left[\sum_{n\ge0}\gamma^n r_{n+1}\,\middle|\, s_0=s\right],
\qquad
V^\star(s)=\sup_\pi V^\pi(s),
\]
\[
Q^\pi(s,a)=\mathbb{E}_\pi\!\left[\sum_{n\ge0}\gamma^n r_{n+1}\,\middle|\, s_0=s,\ a_0=a\right],
\qquad
Q^\star(s,a)=\sup_\pi Q^\pi(s,a).
\]

Under linear realizability, we first show that if $V^\star$ can be represented in a learned feature space by a matrix $\Theta^\star$ of rank $r$, then any orthogonal bottleneck dimension $k\ge r$ is representationally sufficient. In particular, inserting a fixed orthonormal projector $B^\top$ after the encoder does not reduce expressivity relative to the given features. We then prove a trainability equivalence: if $W_t \in \mathbb{R}^{D \times D}$ denotes the encoder's final layer parameters at iteration $t$, gradient descent on the projected parameterization induces the same gradient dynamics on the composite map $A_t=B^\top W_t$ as gradient descent on an explicit $k$-dimensional map $C_t$ when initialized identically. This reduction allows standard results for deep linear networks and matrix factorization to be applied to the projected architecture. For completeness, we restate \Cref{ass:linear} and \Cref{th-main:orthogonal-bottleneck}. We state the analysis in terms of $V^\star$, but the same definitions and arguments apply to $Q^\star$ by replacing $s$ with $(s,a)$ and using action-dependent features.

\begin{assumption}[Linear realizability]
There exists a matrix $\Theta^{\star} \in \mathbb{R}^{m \times D}$ such that for all $s \in \mathcal{S}$,
\[
    V^{\star}(s) = \Theta^{\star}\phi(s).
\]
\end{assumption}

\begin{proposition}
Assume $V^{\star}$ is linearly realizable in feature space with rank $r = \mathrm{rank}(\Theta^{\star})$, and let $H(h;\theta)$ be any head expressive enough to realize at least a linear layer. For any $k \ge r$ and orthonormal $B \in \mathbb{R}^{D \times k}$:
\begin{enumerate}[leftmargin=*,itemsep=2pt]
    \item \textbf{Representational sufficiency.}
    There exist encoder parameters and head parameters $\theta^{\star}$ such that the network
    \[
        s \;\mapsto\; H\bigl(B^\top z(s);\theta^{\star}\bigr)
    \]
    exactly realizes $V^{\star}(s)$ for all $s \in \mathcal{S}$.

    \item \textbf{Trainability:}
    Let $W \in \mathbb{R}^{D \times D}$ be the encoder's final layer and $A_t = B^\top W_t$ the composite feature-to-bottleneck map. Training $(\theta, W)$ by gradient descent on loss $\mathcal{L}$ evolves $A_t$ identically to training the direct parameterization $h = C\phi(s)$ on $(\theta, C)$, given $C_0 = A_0$.
\end{enumerate}
\end{proposition}

\begin{proof}
For general notation, fix a feature map $\phi:\mathcal{S}\to\mathbb{R}^{D}$ as in Assumption~\ref{ass:linear}.
We focus on the last linear layer of the encoder
\begin{equation*}
        z(s) = W\,\phi(s)\in\mathbb{R}^{D},
    \qquad W\in\mathbb{R}^{D\times D}.
\end{equation*}

The fixed orthogonal bottleneck forms
\begin{equation*}
        h(s) = B^\top z(s) = B^\top W\,\phi(s)\in\mathbb{R}^{k}.
\end{equation*}
Define the composite feature-to-bottleneck map $A = B^\top W \in \mathbb{R}^{k\times D}$, so that throughout, $ h(s) = A\,\phi(s)$.
The network output is
\begin{equation*}
        \widehat V(s) \;=\; H\bigl(h(s);\theta\bigr) \;=\; H\bigl(A\phi(s);\theta\bigr).
\end{equation*}

\subsubsection*{1. Representational sufficiency}

We will explicitly construct parameters $(W^\star,\theta^\star)$ such that for all $s$,
\[
    H\bigl(B^\top W^\star \phi(s);\theta^\star\bigr) = \Theta^\star\phi(s) = V^\star(s).
\]
The key point is that $\Theta^\star$ has rank $r$ and we assume $k\ge r$.

Since $\Theta^\star\in\mathbb{R}^{m\times D}$ has rank $r$, it admits a singular value decomposition
\[
    \Theta^\star = U_r \Sigma_r V_r^\top,
\]
where
\begin{itemize}[leftmargin=*,itemsep=2pt]
    \item $U_r\in\mathbb{R}^{m\times r}$ has orthonormal columns ($U_r^\top U_r=I_r$),
    \item $\Sigma_r\in\mathbb{R}^{r\times r}$ is diagonal with strictly positive singular values,
    \item $V_r\in\mathbb{R}^{D\times r}$ has orthonormal columns ($V_r^\top V_r=I_r$).
\end{itemize}

We can further factor $\Theta^*$ through an $r$-dimensional bottleneck by $\Theta^* = L R,$
where
\[
    L = U_r\Sigma_r^{1/2}\in\mathbb{R}^{m\times r},
    \qquad
    R = \Sigma_r^{1/2}V_r^\top\in\mathbb{R}^{r\times D}.
\]

Then, indeed,
\[
    LR = U_r\Sigma_r^{1/2}\Sigma_r^{1/2}V_r^\top = U_r\Sigma_rV_r^\top = \Theta^\star.
\]

Since the bottleneck dimension $k$ satisfies $k \ge r$, we can embed this $r$-dimensional factorization into a $k$-dimensional one by padding with zeros.  Define
\[
    U^\star = \bigl[\,L \;\;\; 0_{m\times(k-r)}\,\bigr]\in\mathbb{R}^{m\times k},
    \qquad
    A^\star =
    \begin{bmatrix}
        R\\[2pt]
        0_{(k-r)\times D}
    \end{bmatrix}\in\mathbb{R}^{k\times D}.
\]
Then
\[
    U^\star A^\star
    = \bigl[\,L\;\;0\,\bigr]
      \begin{bmatrix}
        R\\
        0
      \end{bmatrix}
    = LR
    = \Theta^\star.
\]
Thus $\Theta^\star$ factors through a $k$-dimensional bottleneck.

We now must realize the composite map $A^\star=B^\top W^\star$ for some $W^\star\in\mathbb{R}^{D\times D}$.
Since $B\in\mathbb{R}^{D\times k}$ has orthonormal columns, i.e.\ $B^\top B=I_k$, a canonical choice is
\[
    W^\star = B A^\star \in \mathbb{R}^{D\times D}.
\]
Then
\[
    B^\top W^\star
    = B^\top(BA^\star)
    = (B^\top B)A^\star
    = I_k A^\star
    = A^\star.
\]

By assumption, the head $H(h;\theta)$ contains at least a linear layer. Concretely, this means there exists a subset of parameters inside $\theta$ that can implement an affine map
\[
    h \mapsto U h + b
\]
for arbitrary $U\in\mathbb{R}^{m\times k}$ and $b\in\mathbb{R}^m$, possibly followed and/or preceded by additional transformations.
We will choose $\theta^\star$ so that overall the head implements a pure linear map $h\mapsto U^\star h$.
This is always possible for common heads used in RL (e.g., an MLP head) by setting all subsequent layers to identity (or appropriate weights) and biases to zero; equivalently, one may take the output layer to be linear with weight $U^\star$ and zero bias, and set any intermediate layers to implement the identity on $\mathbb{R}^k$.

Thus we choose $\theta^\star$ such that
\[
    H(h;\theta^\star) = U^\star h
    \qquad\forall\,h\in\mathbb{R}^k.
\]

With these choices, for any $s\in\mathcal{S}$ we have
\[
    H\bigl(B^\top z(s);\theta^\star\bigr)
    = H\bigl(B^\top W^\star\phi(s);\theta^\star\bigr)
    = U^\star (B^\top W^\star)\phi(s)
    = U^\star A^\star\phi(s)
    = \Theta^\star\phi(s)
    = V^\star(s).
\]
This proves representational sufficiency.

\subsection*{2. Trainability}

We now prove that training $(\theta,W)$ with the orthogonal bottleneck induces the same gradient descent dynamics on the composite map $A_t=B^\top W_t$ as training a direct bottleneck map $C_t$ in the parameterization $h=C\phi(s)$, provided $C_0=A_0$.

Consider an arbitrary (differentiable) training objective $\mathcal{L}$ computed from the head output.
For example, $\mathcal{L}$ can be an empirical risk over a dataset or any differentiable surrogate used within an RL update (e.g., value regression loss, PPO value loss, etc.). The only property we use is that $\mathcal{L}$ depends on $W$ only through the bottleneck features $h(s)=B^\top W\phi(s)$ (with $B$ fixed).

Define the projected parameterization
\[
    \widehat V_{\text{proj}}(s;\theta,W)
    = H\bigl(B^\top W\phi(s);\theta\bigr)
    = H\bigl(A\phi(s);\theta\bigr),
    \qquad A=B^\top W,
\]
with a corresponding loss $\mathcal{L}_{\text{proj}}(\theta,W)$. Similarly, define the direct parameterization 
\[
    \widehat V_{\text{dir}}(s;\theta,C)
    = H\bigl(C\phi(s);\theta\bigr),
    \qquad C\in\mathbb{R}^{k\times D}.
\]
with a corresponding loss $\mathcal{L}_{\text{dir}}(\theta,C)$.

By construction, if we identify $C=A=B^\top W$, then the two models produce identical outputs for all $s$ and hence identical loss values with
\[
    \mathcal{L}_{\text{proj}}(\theta,W)
    = \mathcal{L}_{\text{dir}}(\theta, A)
    \quad\text{with }A=B^\top W.
\]
For convenience, define the induced loss
\[
    \widetilde{\mathcal{L}}(\theta,A)
    =
    \mathcal{L}_{\text{dir}}(\theta,A),
\]
so that
\[
    \mathcal{L}_{\text{proj}}(\theta,W) = \widetilde{\mathcal{L}}(\theta,B^\top W).
\]

Fix $\theta$ and consider $W\mapsto \mathcal{L}_{\text{proj}}(\theta,W)=\widetilde{\mathcal{L}}(\theta,B^\top W)$.
Let $dW\in\mathbb{R}^{D\times D}$ be an arbitrary perturbation. The corresponding perturbation of $A=B^\top W$ is
\[
    dA = B^\top dW.
\]
Using the chain rule in differential form,
\[
    d\mathcal{L}_{\text{proj}}
    = d\widetilde{\mathcal{L}}
    = \langle \nabla_A \widetilde{\mathcal{L}}(\theta,A),\, dA\rangle_F.
\]
Substituting $dA=B^\top dW$ results in
\begin{align*}
    d\mathcal{L}_{\text{proj}}
    &= \langle \nabla_A \widetilde{\mathcal{L}}(\theta,A),\, B^\top dW\rangle_F \\
    &= \mathrm{Tr}\!\left((\nabla_A \widetilde{\mathcal{L}})^\top B^\top dW\right) \\
    &= \mathrm{Tr}\!\left(B(\nabla_A \widetilde{\mathcal{L}})^\top dW\right)
     \;=\;
     \mathrm{Tr}\!\left((B\nabla_A \widetilde{\mathcal{L}})^\top dW\right) \\
    &= \langle B\nabla_A \widetilde{\mathcal{L}}(\theta,A),\, dW\rangle_F,
\end{align*}
where in the intermediate step we used the cyclicity of the trace.
Since this holds for all perturbations $dW$, the gradient is
\[
    \nabla_W \mathcal{L}_{\text{proj}}(\theta,W)
    = B\,\nabla_A \widetilde{\mathcal{L}}(\theta,A),
    \qquad A=B^\top W.
\]

Now, consider gradient descent updates on $(\theta,W)$ with step size $\eta>0$:
\[
    \theta_{t+1} = \theta_t - \eta\,\nabla_\theta \mathcal{L}_{\text{proj}}(\theta_t,W_t),
    \qquad
    W_{t+1} = W_t - \eta\,\nabla_W \mathcal{L}_{\text{proj}}(\theta_t,W_t).
\]
Define $A_t=B^\top W_t$. Then
\begin{align*}
    A_{t+1}
    &= B^\top W_{t+1}
     = B^\top\!\left(W_t - \eta\,\nabla_W \mathcal{L}_{\text{proj}}(\theta_t,W_t)\right) \\
    &= B^\top W_t - \eta\,B^\top \nabla_W \mathcal{L}_{\text{proj}}(\theta_t,W_t) \\
    &= A_t - \eta\,B^\top \left(B\,\nabla_A \widetilde{\mathcal{L}}(\theta_t,A_t)\right) \\
    &= A_t - \eta\,(B^\top B)\,\nabla_A \widetilde{\mathcal{L}}(\theta_t,A_t).
\end{align*}
By orthonormality, we have $B^\top B = I_k$, so we obtain
\[
    A_{t+1}
    = A_t - \eta\,\nabla_A \widetilde{\mathcal{L}}(\theta_t,A_t).
\]
But this is exactly the gradient descent update that would be obtained by directly training the parameter $C\in\mathbb{R}^{k\times D}$ in the direct parameterization with the same step size:
\[
    C_{t+1} = C_t - \eta\,\nabla_C \mathcal{L}_{\text{dir}}(\theta_t,C_t)
           = C_t - \eta\,\nabla_A \widetilde{\mathcal{L}}(\theta_t,C_t).
\]
Thus, if $C_0=A_0$, then by induction $C_t=A_t$ for all $t$.

Finally, we verify that the head-parameter updates match under the identification $C_t=A_t$.
Since the two losses satisfy
\[
    \mathcal{L}_{\text{proj}}(\theta,W) = \widetilde{\mathcal{L}}(\theta, B^\top W),
    \qquad
    \mathcal{L}_{\text{dir}}(\theta,C) = \widetilde{\mathcal{L}}(\theta, C),
\]
we have, for fixed $A=B^\top W$ and $C=A$,
\[
    \nabla_\theta \mathcal{L}_{\text{proj}}(\theta,W)
    = \nabla_\theta \widetilde{\mathcal{L}}(\theta,A)
    = \nabla_\theta \mathcal{L}_{\text{dir}}(\theta,A).
\]
Hence the $\theta$-iterates coincide as well when initialized identically.

We have shown that when $k\ge r$, there exist parameters realizing $V^\star$ exactly through the fixed orthogonal bottleneck, and that gradient descent on $(\theta,W)$ induces the same dynamics on the composite map $A_t=B^\top W_t$ (and on $\theta_t$) as direct training of $(\theta,C)$ in the parameterization $h=C\phi(s)$ with $C_0=A_0$. This completes the proof.

\end{proof}

\section{Encoder Width Sweep}
\label{supp:encoder_width}
To disentangle encoder capacity from representation dimensionality, we vary the width of the encoder's final layer on Humanoid while fixing the bottleneck dimension to $k=8$. \Cref{fig:encoder_width} reports the corresponding learning curves. Across a wide range of widths $D$, curves overlap closely, and final performance is similar, indicating weak sensitivity to encoder width once the bottleneck dimension is fixed. That said, very large widths (e.g., $D=1024$) can be slightly worse than moderate widths, suggesting diminishing returns, and occasional mild degradation from increasing encoder capacity beyond what is needed in this setting.

\begin{figure}[htp]
    \centering
    \includegraphics[width=0.55\linewidth]{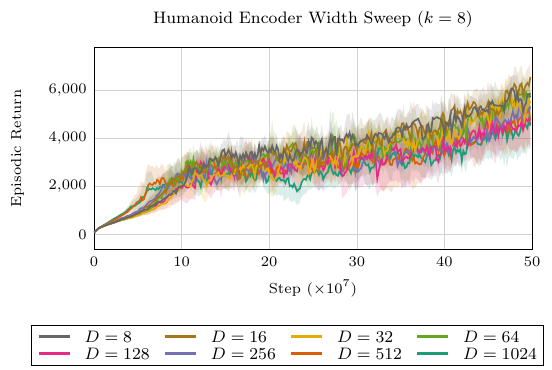}
    \caption{Humanoid encoder-width $D$ sweep with fixed bottleneck dimension $k=8$. Curves largely overlap across encoder widths, indicating that performance is only weakly sensitive to encoder width once $k$ is fixed; very large widths (e.g., $D=1024$) can exhibit slightly lower returns.}
    \label{fig:encoder_width}
\end{figure}

\section{Orthogonal Initialization Variants}
\label{supp:orth_init}
To verify that our results are not sensitive to the orthogonalization method used to initialize the fixed projection matrix $B$, we compare three standard procedures on Humanoid while fixing the bottleneck dimension to $k=18$, which is at or above the recovery threshold for this task.
We consider QR factorization, SVD (using the left singular vectors), and a polar-style normalization, in which we sample a Gaussian matrix $X$ and normalize it by the inverse square root of its Gram matrix, setting $B = X (X^\top X)^{-1/2}$, where $(X^\top X)^{-1/2}$ is computed via an eigen-decomposition with a small floor of $10^{-6}$ for numerical stability.

\Cref{fig:orth_init_humanoid} shows the corresponding learning curves. Across all three initializations, curves overlap closely, and final performance is similar, indicating that the findings are insensitive to the specific method used to construct an orthonormal basis for the bottleneck subspace in this setting.

\begin{figure}[hbtp]
    \centering
    \includegraphics[width=0.55\linewidth]{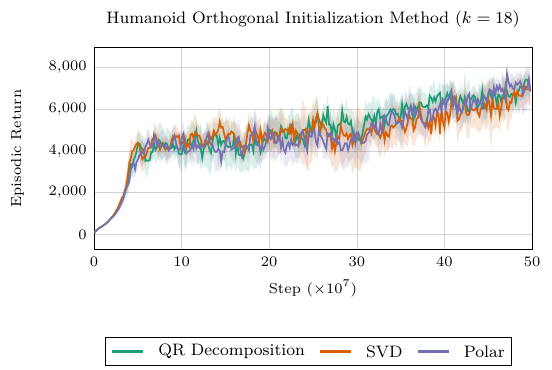}
    \caption{Humanoid learning curves with a fixed bottleneck dimension $k=18$, comparing three orthogonal initialization methods for the fixed projection matrix $B$: QR decomposition, SVD, and a polar-style normalization. Performance is similar across initialization methods, suggesting that results are not driven by the particular orthogonalization procedure.}
    \label{fig:orth_init_humanoid}
\end{figure}

\section{Full Learning Curves}
\label{supp:learning_curves}

This appendix provides full learning curves for Classic Control and MinAtar, complementing the main-text results that focus on bottleneck sweeps and aggregate performance. For Atari and Brax MuJoCo, plotting learning curves for every bottleneck dimension $k$ would be visually cluttered and redundant with the summary in \Cref{fig:atari_mujoco}. Instead, for each environment, we report representative learning curves for the unconstrained baseline, a bottleneck dimension below the recovery threshold, and a bottleneck dimension at or above the recovery threshold. This presentation makes the qualitative transition from failure to baseline-level learning explicit while keeping the figures readable. 

\Cref{fig:cc_curves,fig:minatar_curves} show full learning curves for Classic Control and MinAtar across several bottleneck dimensions. \Cref{fig:atari_curves,fig:mujoco_curves} show representative curves for Atari and Brax MuJoCo, comparing the no-projection baseline to a suboptimal and a recovering bottleneck dimension, with the corresponding $k$ values listed in \Cref{tab:k_map_curves}. Shaded regions indicate 95\% bootstrapped confidence intervals across the 10 seeds.

In addition to learning performance, for each learning-curve figure in this appendix, we include a matched figure that plots the normalized mean effective rank of the value representations over training for the same set of runs and bottleneck dimensions. This makes it possible to compare performance recovery and the evolution of representation rank side-by-side across benchmarks. \Cref{fig:cc_rank} provides the effective rank curves for Classic Control, \Cref{fig:minatar_rank} for MinAtar, \Cref{fig:atari_rank} for Atari, \Cref{fig:brax_rank} for Brax MuJoCo, and \Cref{fig:metaworld_lc_rank} shows a paired two-panel figure with the main-text learning curve and the corresponding effective rank curve for the same runs.

\begin{figure}
    \centering
    \includegraphics[width=0.75\linewidth]{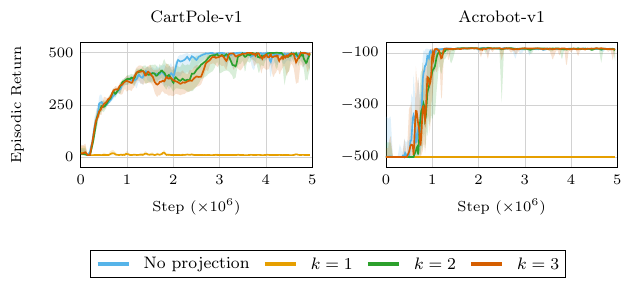}
    \caption{Classic Control learning curves for CartPole-v1 and Acrobot-v1 comparing the unconstrained DQN baseline to fixed orthogonal bottlenecks with varying dimension $k$. In both tasks, performance recovers once $k$ reaches a small threshold, and further increases in $k$ provide little additional benefit.}
\label{fig:cc_curves}

\end{figure}

\begin{figure}
    \centering
    \includegraphics[width=1.0\linewidth]{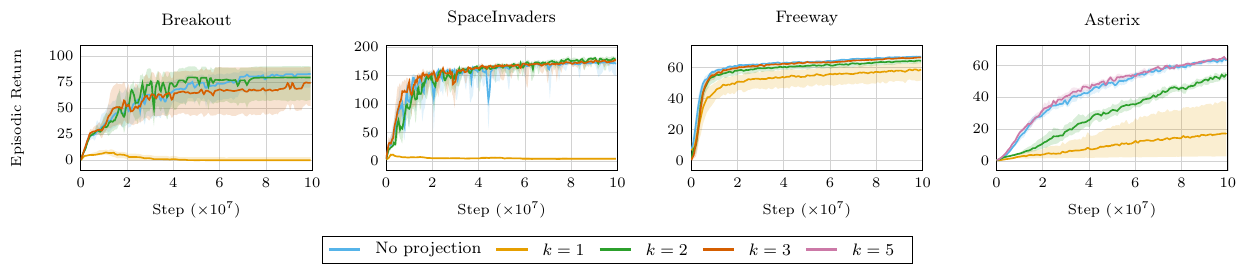}
    \caption{MinAtar learning curves comparing the unconstrained PPO baseline to fixed orthogonal bottlenecks across selected values of $k$. Similar to Classic Control, performance typically recovers once $k$ exceeds a small task-dependent threshold, consistent with the main-text bottleneck sweeps. For Breakout, SpaceInvaders, and Freeway, $k=1$, $k=2$, and $k=3$ are displayed. For Asterix, $k=1$, $k=3$, and $k=5$ are displayed.}
\label{fig:minatar_curves}

\end{figure}

\begin{figure}
    \centering
    \includegraphics[width=1.0\linewidth]{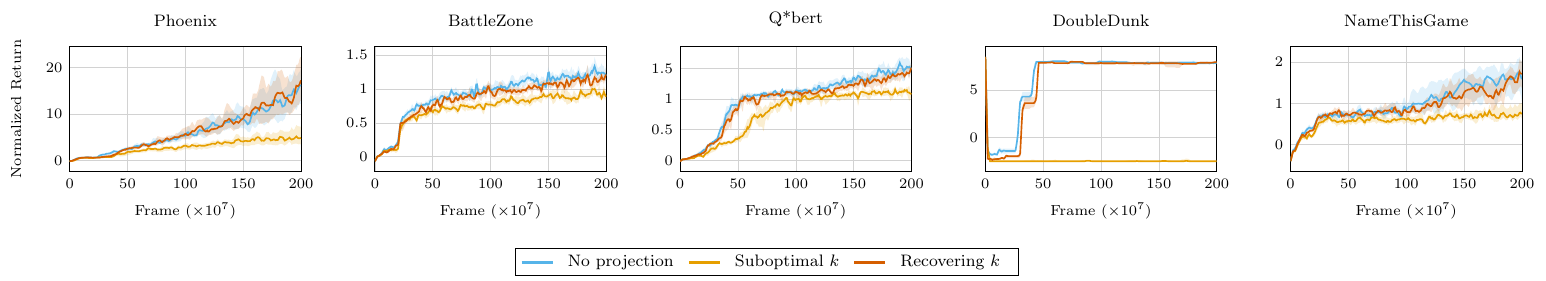}
    \caption{Representative Atari learning curves for each game, showing the unconstrained PQN baseline, a suboptimal bottleneck dimension (Suboptimal $k$), and a bottleneck dimension at/above the recovery threshold (Recovering $k$). Bottleneck dimensions differ by game; the specific $k$ values used are listed in \Cref{tab:k_map_curves}.}
    \label{fig:atari_curves}
\end{figure}

\begin{figure}
    \centering
    \includegraphics[width=1.0\linewidth]{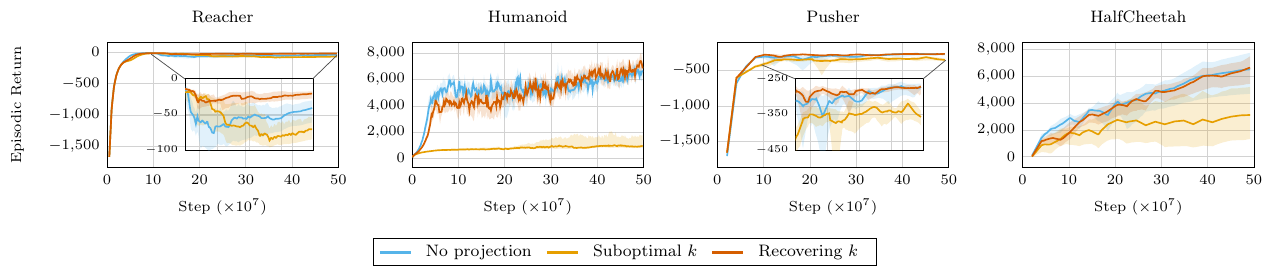}
    \caption{Representative Brax MuJoCo learning curves for each task, showing the unconstrained PPO baseline, a suboptimal bottleneck dimension (Suboptimal $k$), and a bottleneck dimension at/above the recovery threshold (Recovering $k$). Bottleneck dimensions differ by task; the specific $k$ values used are listed in \Cref{tab:k_map_curves}.}
    \label{fig:mujoco_curves}
\end{figure}

\begin{table}[t]
\centering
\small
\caption{Bottleneck dimensions used for the representative learning curves in \Cref{fig:atari_curves} and \Cref{fig:mujoco_curves}. For each environment, we plot the baseline with no projection, a suboptimal bottleneck dimension below the recovery threshold, and a recovering bottleneck dimension at/above the recovery threshold.}
\begin{tabular}{@{}llcc@{}}
\toprule
\textbf{Benchmark} & \textbf{Environment} & \textbf{Suboptimal $k$} & \textbf{Recovering $k$} \\
\midrule
\multirow{5}{*}{Atari-5} 
& Phoenix          & 8 & 128 \\
& Battle Zone      & 4 & 8 \\
& Q*bert           & 8 & 128 \\
& Double Dunk      & 1 & 2 \\
& Name This Game   & 8 & 64 \\
\midrule
\multirow{4}{*}{Brax MuJoCo}
& Reacher          & 2 & 6 \\
& Humanoid         & 6 & 18 \\
& Pusher           & 2 & 16 \\
& HalfCheetah      & 4 & 14 \\
\bottomrule
\end{tabular}
\label{tab:k_map_curves}
\end{table}

\begin{figure}
    \centering
    \includegraphics[width=0.75\linewidth]{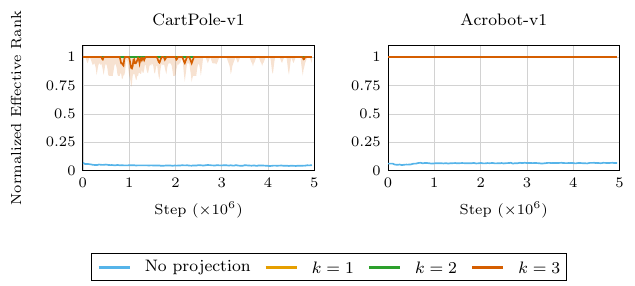}
    \caption{Classic Control normalized mean effective rank curves for CartPole-v1 and Acrobot-v1 for the same runs shown in \Cref{fig:cc_curves}.}
    \label{fig:cc_rank}
\end{figure}

\begin{figure}
    \centering
    \includegraphics[width=1.0\linewidth]{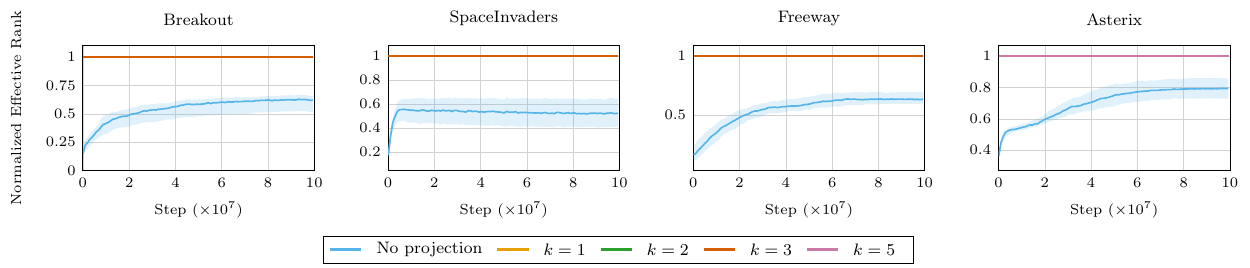}
    \caption{MinAtar normalized mean effective rank curves for the same runs shown in \Cref{fig:minatar_curves}.}
    \label{fig:minatar_rank}
\end{figure}

\begin{figure}
    \centering
    \includegraphics[width=1.0\linewidth]{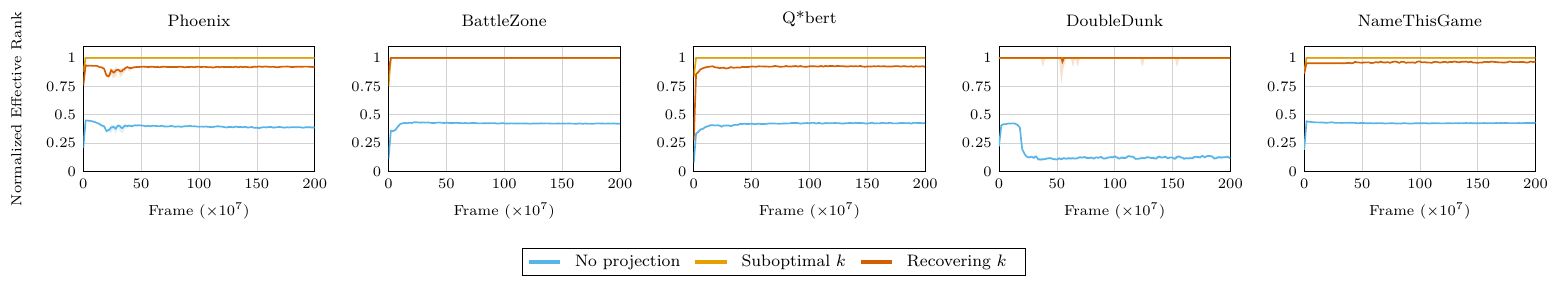}
    \caption{Atari normalized mean effective rank curves for the same runs shown in \Cref{fig:atari_curves}.}
    \label{fig:atari_rank}
\end{figure}

\begin{figure}
    \centering
    \includegraphics[width=1.0\linewidth]{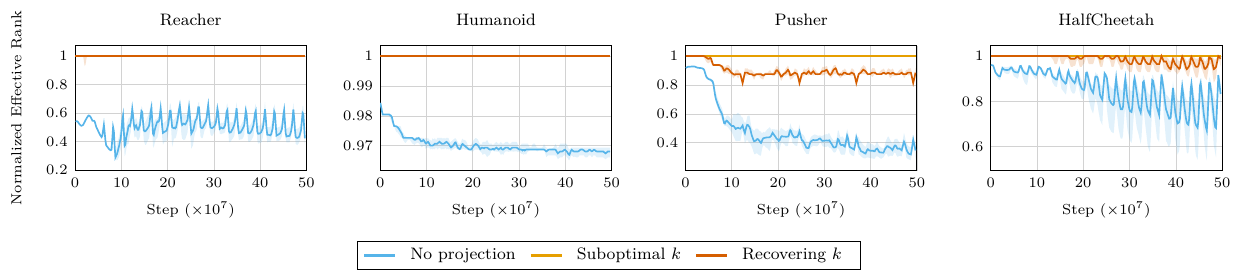}
    \caption{Brax MuJoCo normalized mean effective rank curves for the same runs shown in \Cref{fig:mujoco_curves}.}
    \label{fig:brax_rank}
\end{figure}

\begin{figure}
    \centering
    \begin{subfigure}[t]{0.49\linewidth}
        \centering
        \includegraphics[width=\linewidth]{figures/plot_metaworld.pdf}
        \caption{Learning curve.}
        \label{fig:metaworld_lc}
    \end{subfigure}
    \hfill
    \begin{subfigure}[t]{0.49\linewidth}
        \centering
        \includegraphics[width=\linewidth]{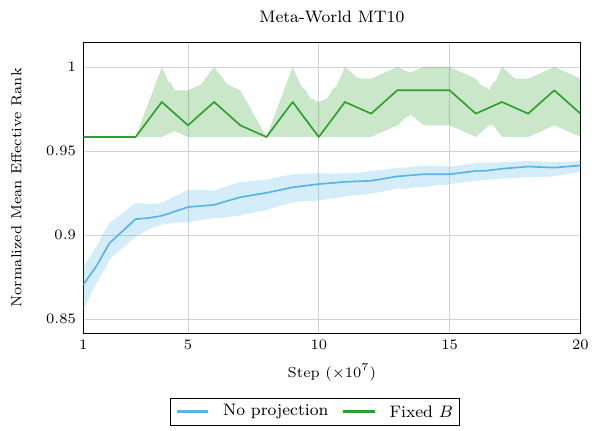}
        \caption{Normalized mean effective rank.}
        \label{fig:metaworld_rank}
    \end{subfigure}
    \caption{Meta-World MT10 performance and normalized mean effective rank for the same runs. \textit{(Left)}: learning curves for the PPO baseline and the fixed orthogonal bottleneck at $k=24$ (as in \Cref{fig:metaworld} in the main text). \textit{(Right)}: the corresponding normalized mean effective rank of the network representations over training.}
    \label{fig:metaworld_lc_rank}
\end{figure}

\end{document}